\documentclass{article} 
\usepackage{iclr2026_delta,times}

\usepackage{amsmath,amssymb,amsthm}
\usepackage{mathtools}
\usepackage{bbm}
\usepackage{dsfont}

\usepackage{graphicx}
\usepackage{booktabs,array}
\usepackage{multirow}
\usepackage{wrapfig}
\usepackage{placeins}

\usepackage{tikz}
\usetikzlibrary{tikzmark}

\usepackage{enumitem}

\usepackage{algorithm}
\usepackage{algorithmic}
\usepackage{eqparbox}

\usepackage{listings}
\usepackage{fancyvrb}
\fvset{fontsize=\small}

\usepackage[dvipsnames]{xcolor}
\definecolor{darkpink}{RGB}{199,21,140}

\newcommand{\bss}{\boldsymbol{s}}

\newcommand{\bsw}{\boldsymbol{w}}
\newcommand{\bsx}{\boldsymbol{x}}
\newcommand{\bsy}{\boldsymbol{y}}
\newcommand{\bsz}{\boldsymbol{z}}

\newcommand{\calS}{{\mathcal{S}}}

\newcommand{\calX}{{\mathcal{X}}}

\newcommand{\bmu}{{\boldsymbol{\mu}}}

\newcommand{\bSigma}{\boldsymbol{\Sigma}}

\theoremstyle{plain}
\newtheorem{thm}{Theorem}[section]

\theoremstyle{definition}

\theoremstyle{remark}

\DeclareMathOperator*{\argmin}{arg\,min}

\newcommand{\normal}{\mathcal{N}}

\def\(#1\){\begin{equation}\begin{aligned}#1\end{aligned}\end{equation}}

\def\sR{{\mathbb{R}}}

\newcommand{\E}{\mathbb{E}}

\newcommand{\del}{\partial}

\usepackage{hyperref}        
\usepackage[nameinlink,capitalise,noabbrev]{cleveref} 

\theoremstyle{plain}

\theoremstyle{definition}

\theoremstyle{remark}

\crefname{thm}{Theorem}{Theorems}

\usepackage{hyperref}
\usepackage{url}
\iclrfinalcopy

\title{Permutation-Symmetrized Diffusion for \\ Unconditional Molecular Generation}

\author{Gyeonghoon Ko, Juho Lee \\
Korea Advanced Institute of Science and Technology \\
\texttt{\{kog, juholee\}@kaist.ac.kr}
}

\begin{document}

\maketitle

\begin{abstract}
Permutation invariance is fundamental in molecular point-cloud generation, yet most diffusion models enforce it indirectly via permutation-equivariant networks on an ordered space. We propose to model diffusion directly on the quotient manifold $\tilde{\calX}=\sR^{d\times N}/S_N$, where all atom permutations are identified. We show that the heat kernel on $\tilde{\calX}$ admits an explicit expression as a sum of Euclidean heat kernels over permutations, which clarifies how diffusion on the quotient differs from ordered-particle diffusion. Training requires a permutation-symmetrized score involving an intractable sum over $S_N$; we derive an expectation form over a posterior on permutations and approximate it using MCMC in permutation space. We evaluate on unconditional 3D molecule generation on QM9 under the EQGAT-Diff protocol, using SemlaFlow-style backbone and treating all variables continuously. The results demonstrate that quotient-based permutation symmetrization is practical and yields competitive generation quality with improved efficiency.
\end{abstract}

\section{Introduction}

Diffusion models provide a flexible framework for generative modeling and have recently achieved strong performance for 3D molecular generation when combined with equivariant architectures \citep{hoogeboom2022equivariant,dunn2024mixed,vignac2023midi,le2023navigating}. A key symmetry of molecules is permutation invariance: relabeling atoms does not change the underlying structure. Most diffusion approaches address this by representing molecules in an ordered space $\calX=\sR^{d\times N}$ and enforcing permutation \emph{equivariance} of both the transition kernels and the score network \citep{zaheer2017deep}. This guarantees that permuting the input permutes the output in the same way, but the diffusion trajectory still evolves in an ordered representation that implicitly tracks particle identities.

We instead formulate diffusion directly on the quotient space $\tilde{\calX}=\sR^{d\times N}/S_N$, where all permutations are identified. We show that the heat kernel on $\tilde{\calX}$ admits an explicit expression as a sum of Euclidean heat kernels over permutations. This characterization highlights a qualitatively different behavior: during noising and denoising, particle identities may be effectively exchanged. The resulting permutation-symmetrized score can be written as an expectation under a posterior distribution over permutations, which we approximate via MCMC. Related symmetrized diffusion ideas were introduced in \citet{anon2026riemannian}, but their approach additionally requires training a physics-informed neural network (PINN). We evaluate on unconditional QM9 generation \citep{ramakrishnan2014quantum} under the EQGAT-Diff framework \citep{le2023navigating}, using a lightweight SemlaFlow-style backbone \citep{irwin2024semlaflow} and treating all variables continuously.

\section{Background}

\subsection{Diffusion models on Riemannian Manifold}

Consider an ambient space $\calX$ in which the data lie. A \textit{noising process} or \textit{forward process} is a stochastic process $\{\bsx^{(t)}\}_{t=0}^T$ starting from $\bsx^{(0)} \sim p_0 = p_{\text{data}}$, usually defined by a stochastic differential equation (SDE) adding noises, with transition kernel $p_{t|s}(\bsx^{(t)} | \bsx^{(s)})$ for $0 \leq s < t \leq T$ and marginal distribution $p_t(\bsx^{(t)}) = \int p_{t|0}(\bsx^{(t)} | \bsx^{(0)}) p_{\text{data}}(\bsx^{(0)}) d\bsx^{(0)}$. A \textit{denoising process} or \textit{reverse process} is a stochastic process designed to reverse the noising process, starting from the noise distribution $q_T(\bsx^{(T)}) \approx p_T$ at time $T$ and aiming to recover the original data distribution $p_\text{data}(\bsx^{(0)})$. The reverse process is also often characterized by a (reverse-time) SDE, with a learnable transition kernel $q_{\theta, s|t}(\bsx^{(s)}|\bsx^{(t)})$ for $s<t$, parametrized by a model parameter $\theta$.

For $\calX=\sR^d$, we use the Ornstein--Uhlenbeck (OU) forward SDE \citep{song2020score}
\begin{equation}\label{eqn:forward-ou}
    d\bsx^{(t)}=-\tfrac12 \bsx^{(t)}dt+d\bsw^{(t)},
\end{equation}
with closed-form Gaussian transitions $p_{t|s}(\bsy\mid\bsx)=\normal(\bsy\mid \bmu_{t|s}(\bsx),\bSigma_{t|s})$ and marginals $p_t$. The associated reverse-time SDE is
\begin{equation}\label{eqn:reverse-ou}
    d\bsx^{(t)}=\Big[-\tfrac12 \bsx^{(t)}-\nabla_{\bsx^{(t)}}\log p_t(\bsx^{(t)})\Big]dt+d\bsw^{(t)}.
\end{equation}
We learn the score with a network $\bss_\theta(\bsx,t)\approx\nabla_{\bsx}\log p_t(\bsx)$ via denoising score matching,
\begin{equation}
\theta^*=\argmin_\theta \E_{t,\bsx^{(0)}}\Big\|\bss_\theta(\bsx^{(t)},t)-\nabla_{\bsx^{(t)}}\log p_t(\bsx^{(t)}\mid \bsx^{(0)})\Big\|^2,
\end{equation}
using the closed-form conditional score of the OU kernel. In practice we may use time-inhomogeneous variants and loss reweighting.

The noising and denoising processes can also be defined well in a non-Euclidean manifold, if the manifold $\calX$ is granted with a Riemannian metric. The Wiener process $\bsw^{(t)}$ is replaced by its Riemannian version $\bsw^{(t)}_\calX$, derived from the Laplace-Beltrami operator $\Delta^{\calX}$ of $\calX$ as an infinitesimal generator. Similar to the Euclidean case, one can define a forward SDE and its reverse
\begin{equation}
    d\bsx^{(t)} = b(\bsx^{(t)})dt + d\bsw^{(t)}_\calX, \quad  d\bsx^{(t)} = - [ b(\bsx^{(t)}) + \nabla_{\bsx^{(t)}} \log p_t(\bsx^{(t)}) ] dt + d\bsw^{(t)}_\calX 
\end{equation}
where $b(\bsx)$ is a tangent vector field on $\calX$. However, unlike the Euclidean case, the closed-form equations for transition kernels and marginal densities are unknown in general, which becomes one of the major challenges of Riemannian diffusion models.

For later use, we also mention that having closed-form expressions for forward transition kernels is possible since the heat kernel $K(t, \bsx,\bsy)$, i.e. the function satisfying $\frac{\del}{\del t}K = \Delta_{\bsy} {K}$ with $\lim_{t \rightarrow 0} K = \delta(\bsx-\bsy)$, is given by the Gaussian
\begin{equation}
    K(t,\bsx,\bsy) = \frac{1}{(4 \pi t)^{d/2}}\exp( -\frac{|\bsx-\bsy|^2}{4t}).
\end{equation}
For the Riemannian manifold $\calX$, the heat kernel $K^{\calX}$ is defined similarly with the Laplacian $\Delta$ replaced by the Laplace-Beltrami operator $\Delta^{\calX}$. The exact form of the heat kernel is unknown in general, and understanding them is a crucial component in modern mathematics.

\section{Method}

\subsection{Rethinking permutation symmetry in diffusion models}

In molecule generation tasks, we usually consider a diffusion process on the space of point clouds $\calX = \sR^{d\times N} = \{\bsx= (x_1,\cdots,x_N) ~|~ x_i \in \sR^N \text{ for } i =1,\cdots,N\}$, with the space $\sR^d$ consisting of coordinates and atom features, e.g. atom types. The molecules $\bsx = (x_1,\cdots,x_N)$ have a permutation symmetry defined by the permutation group $S_N$ which acts by permuting the points as $\sigma(\bsx) = (x_{\sigma^{-1}(1)},\cdots,x_{\sigma^{-1}(N)})$. A common recipe for incorporating the  permutation symmetry is designing a permutation-equivariant transition kernel:
\begin{align}
    p_{t|s}(\bsx^{(t)} | \bsx^{(s)}) = p_{t|s}(\sigma(\bsx^{(t)}) | \sigma(\bsx^{(s)})), \quad
    q_{\theta, s|t}(\bsx^{(s)} | \bsx^{(t)}) = q_{\theta, s|t}(\sigma(\bsx^{(s)}) | \sigma(\bsx^{(t)}))
\end{align}
for all $s<t$ and $\sigma \in S_N$. In this formulation, the equivariance is achieved by using a permutation equivariant diffusion kernel (forward) and a permutation equivariant architecture (reverse). This formulation implies that if the input is altered by a permutation, then the output is altered accordingly by the same permutation, i.e. \textbf{the process behaves identically} under permutation.

On the other hand, we can think of a stronger notion of utilizing the symmetry. Instead of building an equivariant diffusion on $\calX = \sR^{d\times N}$, we model the diffusion process on the quotient manifold $\tilde{\calX} = \sR^{d\times N} / S_N$. Let $\pi : {\calX} = \sR^{d\times N} \rightarrow \tilde{\calX} = \sR^{d\times N} / S_N $ denote the quotient map, and let's write $\tilde{\bsx} = \pi(\bsx)$ for brevity. In this formulation, we have simply $\tilde{\sigma(\bsx)} = \tilde{\bsx}$, i.e. \textbf{the elements are identical} under permutation. The manifold $\tilde{\calX} = \sR^{d\times N} / S_N$ inherits the Euclidean metric, and hence has a well-defined heat kernel. The exact form of the heat kernel is given by \cref{thm:perm-heat}:
\begin{thm}\label{thm:perm-heat}
    The heat kernel $K^{\tilde{\calX}}(t,\tilde{\bsx},\tilde{\bsy})$ associated to the quotient manifold $\tilde{\calX} = \sR^{d\times N} / S_N$ is given by, for $\tilde{\bsx},\tilde{\bsy} \in \tilde{\calX}$,
    \begin{align}
        &K^{\tilde{\calX}}(t,\tilde{\bsx}, \tilde{\bsy})  \nonumber \\
        &\quad = \frac{1}{(4\pi t)^{dN/2}} \sum_{\sigma\in S_N}  \exp(-\frac{|\bsx - \sigma(\bsy)|^2}{4t})  \nonumber 
        = \frac{1}{(4\pi t)^{dN/2}} \sum_{\sigma\in S_N} \prod_{i=1}^N \exp(-\frac{(x_i - y_{\sigma^{-1}(i)})^2}{4t}),
    \end{align}
    i.e. the heat kernel is simply  the sum of the Euclidean heat kernels over all permutations.
\end{thm}
The proof of \cref{thm:perm-heat} is postponed to \cref{app:perm-heat}, but it just comes from the fact that permutations are isometries of Euclidean spaces. 

Now, with the heat kernel identified, it looks much clearer how the diffusion processes in the quotient manifold behave. The classic diffusion process in the Euclidean space inherently assumes that first particle $x_1$ diffuses to the first one $y_1$, second one $x_2$ diffuses to $y_2$, and so on. However, our diffusion opens the possibilities that first particle $x_1$ can be diffused to any $y_i$ for $i \in \{1,\cdots,N\}$, and so on, over all the configurations $\sigma \in S_N$. Intuitively, in this diffusion process, the identities of particles are all shared, so their identities can be swapped with each other while noising or denoising. This behavior resembles the behaviors of \textit{bosons} in particle physics \citep{feynman1979path}.

\subsection{SDEs on the quotient manifold}

Let's go back to the OU process as defined in \cref{eqn:forward-ou} and \cref{eqn:reverse-ou}.
We'll consider the pushforward of the OU process by the quotient map $\pi:\calX \mapsto \tilde{\calX}$. Although working with pushforward process in arbitrary Riemannian manifolds would involve extremely complicated computations in general, the computations are significantly simplified in our case, because the OU process is invariant under permutations and permutations are isometries of the Euclidean space.

The forward and reverse equations of the pushforward of the OU process are given by 
\begin{equation}
    d\tilde{\bsx}^{(t)} = -\frac{1}{2}\tilde{\bsx}^{(t)}dt + d\tilde{\bsw}^{(t)}, \quad
    d\tilde{\bsx}^{(t)} = \left[-\frac{1}{2}\tilde{\bsx}^{(t)} - \nabla_{\tilde{\bsx}^{(t)}}\log \tilde{p}_t(\tilde{\bsx}^{(t)})\right]dt + d\tilde{\bsw}^{(t)}.
\end{equation}
where we use $\tilde{p}_t$ and $\tilde{p}_{s|t}$ for marginals and transition kernels. Rigorously speaking, the term $\tilde{\bsx}^{(t)}$ multiplied to $dt$ is a vector field $\tilde{\bsx}^{(t)} = \tilde{V}(\tilde{\bsx}^{(t)})$ on $\tilde{\calX}$ defined by quotient-mapping the vector field on Euclidean space $V(\bsx)=\bsx$, i.e. $\tilde{V} = \pi \circ V \circ \pi^{-1}$, well-defined since the vector field $V$ is invariant under permutation. Similarly, the gradient $\nabla_{\tilde{\bsx}^{(t)}}$ is in fact the gradient operator on the Riemannian manifold $\tilde{\calX}$, but it locally behaves identical to the Euclidean gradient since the permuatations are isometries.

The transition kernels $\tilde{p}_{t|s}$ and marginals $\tilde{p}_{t}$ can simply yielded by summing the original transition kernels $p_{t|s}$ and marginals $p_{t}$ over all permutations:
\begin{align}
    \tilde{p}_{t|s}(\tilde{\bsy}|\tilde{\bsx}) = \sum_{\sigma \in {S_N}} p_{t|s}(\sigma(\bsy)|\bsx), \quad
    \tilde{p}_{t}(\tilde{\bsy}|\tilde{\bsx}) = \sum_{\sigma \in {S_N}} p_{t}(\sigma(\bsy)|\bsx).
\end{align}

\subsection{Training permutation symmetrized diffusion}

The transition kernel $p_t$ derived in \cref{thm:perm-heat} contains summmation over the permutation group $S_N$, hence naive computation would require at least $O(N!)$ computations. Instead, we look at the score function $\nabla \log \tilde{p}_t$, which will be the target function of diffusion loss. Letting $I(\sigma) = -\frac{|\bsx - \sigma(\bsy)|^2}{4t}$, we have
\begin{align}
    &\nabla_{\tilde{\bsy}} \log p_t (\tilde{\bsy} | \tilde{\bsx})  \nonumber
    = \nabla_{\bsy} \log \sum_{\sigma\in S_N} \exp (I(\sigma))  \nonumber
     = \sum_{\sigma\in S_N} \frac{\exp (I(\sigma))}{\sum_{\sigma'\in S_N} \exp (I(\sigma'))} \nabla_{\bsy} I(\sigma).
\end{align}
Thus, if we consider a distribution $\calS$ on the permutation group $S_N$ whose probability mass function $q(\sigma)$ is proportional to $\exp I(\sigma)$, we have
\begin{equation}
    \nabla_{\bsy} \log p_t (\tilde{\bsx} | \tilde{\bsy}) = \E_{\calS} [\nabla_{\bsy} I(\sigma)].
\end{equation}
Note that $\nabla_{\bsy}I(\sigma) = \frac{\bsx - \sigma(\bsy)}{2t}$ would be the score function for the original Euclidean diffusion model, associated with the original data $\bsx$ and its noised version $\sigma(\bsy)$. We can interpret the probability $q(\sigma) = q(\sigma | \bsx, \bsy)$ as the posterior of the permutation $\sigma$ given the original and noised data, and the target function is the expected value of the original score function over the posterior $q(\sigma)$. In fact, from the perspective of variational inference, learning towards this target function can be seen as minimizing the lower bound for the evidence lower bound (ELBO) of this diffusion model. This argument is explained in detail in \cref{app:bayes}.

Sampling permutations from $\calS$ can be done using Markov chain Monte Carlo (MCMC). Define the cost matrix $C = (C_{ij})$ with $C_{ij} = -\frac{(x_i - y_j)^2}{4t}$. We use an MCMC starting from $\sigma_0 = \text{id} \in S_N$ and the proposals yielded by swapping entries of $i,j$, for $i \in \{1,\cdots,N\}$ sampled uniformly at random and $j$ sampled from distribution proportional to $\exp(C_{ij})$. Once the permutations $\sigma_1,\cdots,\sigma_K \sim \calS$ are sampled, penalizing the model $s_\theta$ towards $\E_{\calS} [\nabla_{\bsy} I(\sigma)] \approx \sum_{k=1}^K \nabla_{\bsy} I(\sigma_k)$ gives an unbiased estimate of the gradient:
\begin{equation}
\begin{split}
&\nabla_\theta \big\| s_{\theta} - \E_{\calS} [\nabla_{\bsy} I(\sigma)] \big\| ^2 
= \E_{\sigma_1,\cdots,\sigma_K \sim \calS} \bigg[\nabla_\theta \big\| s_{\theta} - \sum_{k=1}^K \nabla_{\bsy} I(\sigma) \big\| ^2 \bigg]
\end{split}
\end{equation}
since derivatives of quadratic functions are linear.

\section{Experiment}

We evaluate on unconditional 3D molecule generation on QM9 \citep{ramakrishnan2014quantum}. Our experimental pipeline follows EQGAT-Diff \citep{le2023navigating} (data split, training schedule, sampling procedure, and evaluation), but we replace its backbone with the lighter SemlaFlow architecture \citep{irwin2024semlaflow} that does not carry edge features for efficiency. Throughout, we treat all variables as continuous (i.e., no discrete diffusion over atom types). When computing the permutations using MCMC, edge features are ignored since they will require heavier computation.

Each molecule is represented as a point cloud with 3D coordinates and continuous node features. We train on the QM9 training split and generate samples from the learned reverse-time dynamics following EQGAT-Diff.
We compare primarily against EQGAT-Diff under its default QM9 setup. All methods are evaluated using the same split and metric implementation as in EQGAT-Diff and its prior works \citep{vignac2023midi}.
We report the standard QM9 generation metrics used in EQGAT-Diff, including atom stability, molecule stability, validity, uniqueness, and novelty. The experiments are performed five times, and their averages and standard deviations are reported in \cref{tab:qm9_gen}. Our model wins all the stability/validity metrics by a small margin, and achieves much higher novelty.

\begin{table}[t]
\centering
\footnotesize
\caption{Comparison of QM9 unconditional generation metrics (\%).}
\label{tab:qm9_gen}
\begin{tabular}{lccccc}
\toprule
Method & Atom-Stab & Mol-Stab & Valid & Uniq & Nov \\
\midrule
EQGAT-Diff & $99.87 \pm 0.02$ & $98.24 \pm 0.11$ & $98.58 \pm 0.13$ & $\mathbf{100 \pm 0.00}$ & $61.62 \pm 1.55$ \\
Ours       & $\mathbf{99.89 \pm 0.02}$ & $\mathbf{98.48 \pm 0.27}$ & $\mathbf{98.74 \pm 0.25}$ & $\mathbf{100 \pm 0.00}$ & $\mathbf{67.72 \pm 1.55}$ \\

\bottomrule
\end{tabular}
\end{table}

\section{Conclusion}

We presented a quotient-manifold formulation of permutation symmetry for diffusion models on point clouds, focusing on $\tilde{\calX}=\sR^{d\times N}/S_N$. The resulting heat kernel admits a simple permutation-sum expression, and the corresponding symmetrized score can be estimated efficiently by MCMC over permutations. On QM9, following EQGAT-Diff framework \citep{le2023navigating} with a lighter SemlaFlow architecture \citep{irwin2024semlaflow}, we show the approach is practical in a fully continuous setting. Future work includes improving permutation-sampler efficiency and extending the same quotient perspective to other symmetry quotients in geometric generative modeling.

\newpage
\bibliography{iclr2026_delta}
\bibliographystyle{iclr2026_delta}

\newpage
\appendix
\appendix
\section{Mathematical Details}
\subsection{Proof of \Cref{thm:perm-heat}}\label{app:perm-heat}

\begin{proof}
Let $K^{\calX}(t,\bsx,\bsy)$ be the Euclidean heat kernel on
$\calX\simeq\sR^{dN}$:
\[
K^{\calX}(t,\bsx,\bsy)
= \frac{1}{(4\pi t)^{\frac{dN}{2}}}\exp\!\Big(-\frac{\|\bsx-\bsy\|^2}{4t}\Big),
\qquad t>0.
\]
Define, for representatives $\bsx,\bsy\in\calX$ of $\tilde{\bsx}=\pi(\bsx)$ and
$\tilde{\bsy}=\pi(\bsy)$,
\[
\widehat K(t,\bsx,\bsy)\;:=\;\sum_{\sigma\in S_N} K^{\calX}(t,\bsx,\sigma(\bsy)).
\]

\paragraph{Step 1: $\widehat K$ descends to the quotient.}
Each permutation $\tau\in S_N$ is an isometry of $\calX$, hence
$\|\tau(\bsx)-\tau(\bsy)\|=\|\bsx-\bsy\|$ and therefore
$K^{\calX}(t,\tau(\bsx),\tau(\bsy))=K^{\calX}(t,\bsx,\bsy)$.
Using this and the change of index $\sigma'=\tau^{-1}\sigma$, we get
\[
\widehat K(t,\tau(\bsx),\bsy)
= \sum_{\sigma\in S_N} K^{\calX}(t,\tau(\bsx),\sigma(\bsy))
= \sum_{\sigma'\in S_N} K^{\calX}(t,\bsx,\sigma'(\bsy))
= \widehat K(t,\bsx,\bsy),
\]
and similarly $\widehat K(t,\bsx,\tau(\bsy))=\widehat K(t,\bsx,\bsy)$.
Hence $\widehat K(t,\bsx,\bsy)$ depends only on the orbits
$\tilde{\bsx}=\pi(\bsx)$ and $\tilde{\bsy}=\pi(\bsy)$, so we may define
\[
K^{\tilde{\calX}}(t,\tilde{\bsx},\tilde{\bsy})
:=\widehat K(t,\bsx,\bsy).
\]

\paragraph{Step 2: heat equation.}
For each fixed $\sigma$, the function $\bsx\mapsto K^{\calX}(t,\bsx,\sigma(\bsy))$
solves $(\partial_t-\Delta_{\calX})u=0$ for $t>0$, hence so does the finite sum
$\widehat K(t,\bsx,\bsy)$. Since $\widehat K$ is $S_N$-invariant in $\bsx$,
it corresponds to a function on $\tilde{\calX}$, and thus
$K^{\tilde{\calX}}(\cdot,\cdot,\tilde{\bsy})$ solves the heat equation on the quotient.

\paragraph{Step 3: initial condition.}
Let $\tilde f$ be a smooth test function on $\tilde{\calX}$ and let
$f:=\tilde f\circ \pi$ be its $S_N$-invariant lift to $\calX$.
With respect to the Riemannian volume on $\tilde{\calX}$, the heat semigroup defined by
$K^{\tilde{\calX}}$ satisfies
\begin{align*}
\int_{\tilde{\calX}} K^{\tilde{\calX}}(t,\tilde{\bsx},\tilde{\bsy})\,\tilde f(\tilde{\bsy})\,d\tilde{\bsy}
&= \frac{1}{|S_N|}\int_{\calX} \Big(\sum_{\sigma\in S_N} K^{\calX}(t,\bsx,\sigma(\bsy))\Big)\, f(\bsy)\, d\bsy \\
&= \int_{\calX} K^{\calX}(t,\bsx,\bsz)\, f(\bsz)\, d\bsz,
\end{align*}
where the last equality uses the change of variables $\bsz=\sigma(\bsy)$ and the invariance
$f\circ\sigma=f$ (each term in the sum is identical, cancelling the factor $|S_N|$).
As $t\downarrow 0$, the right-hand side converges to $f(\bsx)=\tilde f(\tilde{\bsx})$
by the standard delta initial condition for the Euclidean heat kernel.
Thus $K^{\tilde{\calX}}$ has the correct initial condition on the quotient.

\paragraph{Conclusion.}
We have shown that the expression in \Cref{thm:perm-heat} defines a kernel on $\tilde{\calX}$
which solves the heat equation for $t>0$ and converges to the delta distribution at $t=0$.
By uniqueness of the heat kernel, it must equal the heat kernel on $\tilde{\calX}$.
\end{proof}

\subsection{Bayesian / variational interpretation of permutation symmetrization}
\label{app:bayes}

\paragraph{Posterior over permutations.}
From \Cref{thm:perm-heat}, the symmetrized transition density on the quotient satisfies
\[
\tilde p_t(\tilde{\bsy}\mid \tilde{\bsx})
\;\propto\;\sum_{\sigma\in S_N}\exp\!\big(I(\sigma)\big),
\qquad
I(\sigma):=-\frac{\|\bsx-\sigma(\bsy)\|^2}{4t}.
\]
Introduce a latent permutation $\sigma$ with a uniform prior $p(\sigma)=1/|S_N|$ and define the
(joint, unnormalized) model
\[
\tilde p_t(\tilde{\bsy},\sigma\mid \tilde{\bsx}) \;\propto\; \exp(I(\sigma)).
\]
Then the Bayes posterior is exactly the ``soft assignment'' used in the main text:
\[
p(\sigma\mid \tilde{\bsx},\tilde{\bsy})
=\frac{\exp(I(\sigma))}{\sum_{\sigma'\in S_N}\exp(I(\sigma'))}
=:q(\sigma\mid \bsx,\bsy).
\]
Consequently (Fisher's identity / mixture-score identity),
\[
\nabla_{\bsy}\log \tilde p_t(\tilde{\bsy}\mid \tilde{\bsx})
=\E_{\sigma\sim q(\sigma\mid \bsx,\bsy)}\big[\nabla_{\bsy} I(\sigma)\big],
\]
which is the permutation-symmetrized score target.

\paragraph{ELBO view.}
For any distribution $r(\sigma)$ on $S_N$, Jensen's inequality gives
\begin{align*}
\log \tilde p_t(\tilde{\bsy}\mid \tilde{\bsx})
&= \log \sum_{\sigma}\exp(I(\sigma)) + \text{const} \\
&= \log \E_{\sigma\sim r}\!\left[\frac{\exp(I(\sigma))}{r(\sigma)}\right] + \text{const} \\
&\ge \E_{r}[I(\sigma)] + H(r) + \text{const}
\;=:\; \mathrm{ELBO}(r),
\end{align*}
and one can write the exact decomposition
\[
\log \tilde p_t(\tilde{\bsy}\mid \tilde{\bsx})
= \mathrm{ELBO}(r) + D_{\mathrm{KL}}\!\big(r(\sigma)\,\|\,q(\sigma\mid \bsx,\bsy)\big).
\]
Thus maximizing the evidence on the quotient is equivalent to maximizing this ELBO, and the
optimal variational distribution is the posterior $r^\star=q(\sigma\mid \bsx,\bsy)$.
Our training objective uses the corresponding posterior expectation of the Euclidean score;
in practice we approximate $\E_{q}[\cdot]$ by MCMC samples of $\sigma$.

\end{document}